\documentclass{article}

\usepackage{amsthm,amsmath,amsfonts,amssymb}
\usepackage{float}
\newtheorem{definition}{Definition}[section]
\newtheorem{theorem}{Theorem}[section]

\newtheorem{lemma}[theorem]{Lemma}
\newtheorem{prop}[theorem]{Proposition}
\usepackage{natbib}
\usepackage{microtype}
\usepackage{subcaption}
\usepackage{graphicx}
\usepackage{dblfloatfix}
\usepackage{booktabs} 

\usepackage{paralist}
\usepackage{tikz}
\usetikzlibrary{cd}

\DeclareMathOperator{\tr}{tr}
\DeclareMathOperator*{\argmin}{arg\,min}

\usepackage{hyperref}

\usepackage[accepted]{icml2019}

\icmltitlerunning{Loss Landscapes of Regularized Linear Autoencoders}

\begin{document}

\twocolumn[
\icmltitle{Loss Landscapes of Regularized Linear Autoencoders}



\icmlsetsymbol{equal}{*}

\begin{icmlauthorlist}
\icmlauthor{Daniel Kunin}{equal,stanford}
\icmlauthor{Jonathan M. Bloom}{equal,broad}
\icmlauthor{Aleksandrina Goeva}{broad}
\icmlauthor{Cotton Seed}{broad}
\end{icmlauthorlist}

\icmlaffiliation{stanford}{Institute for Computational and Mathematical Engineering, Stanford University, Stanford, California, USA}

\icmlaffiliation{broad}{Broad Institute of MIT and Harvard, Cambridge, Massachusetts, USA}

\icmlcorrespondingauthor{Daniel Kunin}{kunin@stanford.edu}
\icmlcorrespondingauthor{Jonathan M. Bloom}{jbloom@broadinstitute.org}

\icmlkeywords{linear, autoencoder, regularization, orthogonality, critical point, loss function, loss surface, loss landscape, principal components analysis, singular value decomposition, grassmannian, morse thoery, morse homology, perfect morse function}

\vskip 0.3in]



\printAffiliationsAndNotice{\icmlEqualContribution} 

\begin{abstract}
Autoencoders are a deep learning model for representation learning. When trained to minimize the distance between the data and its reconstruction, linear autoencoders (LAEs) learn the subspace spanned by the top principal directions but cannot learn the principal directions themselves. In this paper, we prove that $L_2$-regularized LAEs are symmetric at all critical points and learn the principal directions as the left singular vectors of the decoder. We smoothly parameterize the critical manifold and relate the minima to the MAP estimate of probabilistic PCA. We illustrate these results empirically and consider implications for PCA algorithms, computational neuroscience, and the algebraic topology of learning.
\end{abstract}

\section{Introduction}
\label{introduction}

Consider a data set consisting of points $x_1, \ldots, x_n$ in $\mathbb{R}^m$. Let $X \in \mathbb{R}^{m \times n}$ be the data matrix with columns $x_i$. We will assume throughout that $k \le\min\{m,n\}$ and that the singular values of $X$ are positive and distinct.

An \textit{autoencoder} consists of an encoder $f: \mathbb{R}^m \to \mathbb{R}^k$ and decoder $g: \mathbb{R}^k \to \mathbb{R}^m$; the latter maps the latent representation $f(x_i)$ to the reconstruction $\hat{x}_i = g(f(x_i))$ \citep{goodfellow}. The full network is trained to minimize reconstruction error, typically the squared Euclidean distance between the dataset $X$ and its reconstruction $\hat{X}$ (or equivalently, the Frobenius norm of $X - \hat{X}$). When the activations of the network are the identity, the model class reduces to that of one encoder layer $W_1 \in \mathbb{R}^{k \times m}$ and one decoder layer $W_2 \in \mathbb{R}^{m \times k}$. We refer to this model as a \textit{linear autoencoder} (LAE) with loss function defined by

$$\mathcal{L}(W_1, W_2) = ||X - W_2 W_1 X||_F^2.$$

Parameterizing $\mathcal{L}$ by the product $W = W_2 W_1$, the Eckart-Young Theorem \citep{young} states that the optimal $W$ orthogonally projects $X$ onto the subspace spanned by its top $k$ \textit{principal directions}\footnote{The principal directions of $X$ are the eigenvectors of the $m \times m$ covariance of $X$ in descending order by eigenvalue, or equivalently the left singular vectors of the mean-centered $X$ in descending order by (squared) singular values.}.

Without regularization, LAEs learn this subspace but cannot learn the principal directions themselves due to the symmetry of $\mathcal{L}$ under the action of the group $\mathrm{GL}_k(\mathbb{R})$ of invertible $k \times k$ matrices defined by $(W_1, W_2) \mapsto (G W_1, W_2 G^{-1})$:

\begin{equation}
\label{symg}
X - (W_2 G^{-1})(G W_1) X = X - W_2 W_1X.
\end{equation}

Indeed, $\mathcal{L}$ achieves its minimum value on a smooth submanifold of $\mathbb{R}^{k \times m} \times \mathbb{R}^{m \times k}$ diffeomorphic to $\mathrm{GL}_k(\mathbb{R})$; the learned latent representation is only defined up to deformation by invertible linear maps; and the $k$-dimensional eigenspace of $W$ with eigenvalue one has no preferred basis.

In light of the above, the genesis of this work was our surprise at the theorem\footnote{The theorem was retracted following our correspondance.} and empirical observation in \citet{plaut} that the principal directions of $X$ are recovered from a trained LAE as the left singular vectors of the decoder (or as the right singular vectors of the encoder). We realized by looking at the code that training was done with the common practice of $L_2$-regularization:

$$\mathcal{L}_\sigma(W_1, W_2) = \mathcal{L}(W_1, W_2) + \lambda\left(||W_1||^2_F + ||W_2||^2_F\right).$$

In this paper, we prove that LAEs \textit{with} $L_2$-regularization do in fact learn the principal directions in this way, while shrinking eigenvalues in the manner of probabilistic PCA \citep{bishop99}. The key idea is that regularization reduces the symmetry group from $\mathrm{GL}_k(\mathbb{R})$ to the orthogonal group $\mathrm{O}_k(\mathbb{R})$, which preserves the structure of SVD. We further prove that the encoder and decoder are transposes at all critical points, with implications for whether the brain could plausibly implement error backpropagation.

\begin{figure*}[t]
\label{vectorfields}
\vskip 0.2in
\begin{center}
\begin{subfigure}{0.49\columnwidth}
  \centerline{
  \includegraphics[width=\linewidth]{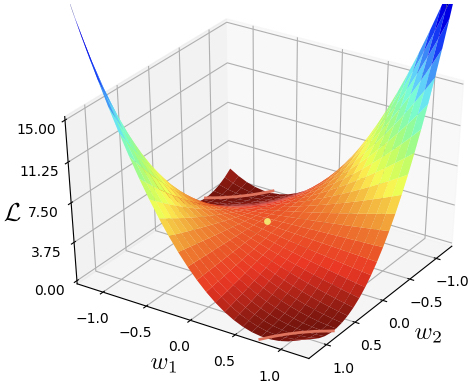}
  }
  \caption{Unregularized}
\end{subfigure}
\begin{subfigure}{0.49\columnwidth}
  \centerline{
  \includegraphics[width=\linewidth]{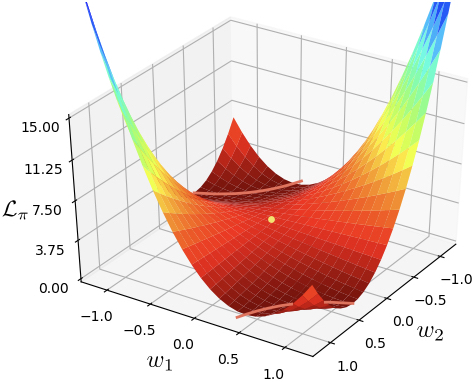}
  }
  \caption{Product ($\lambda = 2$)}
\end{subfigure}
\begin{subfigure}{0.49\columnwidth}
  \centerline{
  \includegraphics[width=\linewidth]{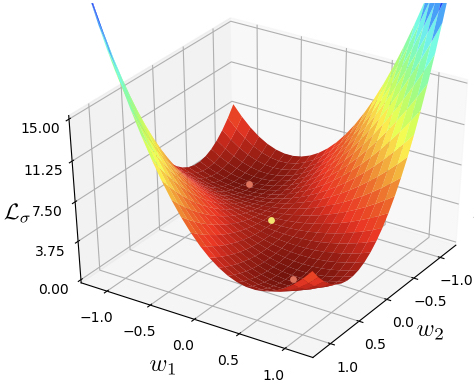}
  }
  \caption{Sum ($\lambda = 2$)}
\end{subfigure}
\begin{subfigure}{0.49\columnwidth}
  \centerline{
  \includegraphics[width=\linewidth]{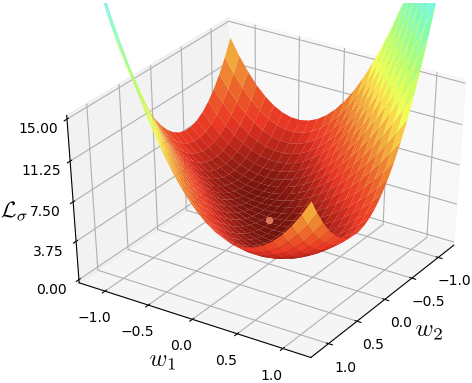}
  }
  \caption{Sum ($\lambda = 4$)}
\end{subfigure}
\caption{Scalar loss landscapes with $x^2=4$. Yellow points are saddles and red curves and points are global minima.}
\label{one dimensional}
\end{center}
\vskip -0.2in
\end{figure*}

\subsection{Related work}

Building on the original work of \citet{young} on low-rank matrix approximation, \citet{izenman} demonstrated a connection between a rank-reduced regression model similar to an LAE and PCA. \citet{bourlard} characterized the minima of an unregularized LAE; \citet{baldi1989} extended this analysis to all critical points.

Several studies of the effect of regularization on LAEs have emerged of late. The rank-reduced regression model was extended in \citet{mukherjee} to the study of rank-reduced ridge regression. A similar extension of the LAE model was given in \citet{josse}. An in depth analysis of the linear denoising autoencoder was given in \citet{pretorius} and most recently, \citet{mianjy} explored the effect of dropout regularization on the minima of an LAE.

While $L_2$-regularization is a foundational technique in statistical learning, its effect on autoencoder models has not been fully characterized. Recent work of \citet{mehta} on $L_2$-regularized deep linear networks applies techniques from algebraic geometry to highlight how algebraic symmetries result in ``flat" critical manifolds and how $L_2$-regularization breaks these symmetries to produce isolated critical points. We instead apply elementary linear algebra, with intuition from algebraic topology, to completely resolve dynamics in the special case of LAEs.

\subsection{Our contributions}

The contributions of our paper are as follows. 
\begin{compactitem}
    \item In Section \ref{secloss} we consider LAEs with (i) no regularization, (ii) $L_2$-regularization of the composition of the encoder and decoder, and (iii) $L_2$-regularization of the encoder and decoder separately as in $\mathcal{L}_\sigma$. We build intuition by analyzing the scalar case, consider the relationship between regularization and orthogonality, and deduce that the encoder and decoder are transposes at all critical points of $\mathcal{L}_\sigma$.
    \item In Section \ref{bayesian models}, we realize all three LAE models as generative processes, most notably relating the minimum of $\mathcal{L}_\sigma$ and the MAP estimate of  probabilistic PCA.
    \item In Section \ref{seccrit}, we characterize all three loss landscapes. To build intuition, we first leave the overparameterized world of coordinate representations to think geometrically about the squared distance from a plane to a point cloud. We expand on this topological viewpoint in Appendix \ref{morse}.
    \item In Section \ref{secemp}, we illustrate these results empirically, with all code and several talks available on \href{https://github.com/danielkunin/Regularized-Linear-Autoencoders}{GitHub}\footnote{\href{https://github.com/danielkunin/Regularized-Linear-Autoencoders}{github.com/danielkunin/Regularized-Linear-Autoencoders}}.
    \item In Section \ref{secdisc}, we discuss implications for eigendecomposition algorithms, computational neuroscience, and deep learning.
\end{compactitem}
The connections we draw between regularization and orthogonality, LAEs and probabilistic PCA, and the topology of Grassmannians are novel and provide a deeper understanding of the loss landscapes of regularized linear autoencoders.

\section{Regularized LAEs}
\label{secloss}

In the Appendix \ref{bias and mean centering}, we provide a self-contained derivation of the fact that an LAE with bias parameters is equivalent to an LAE without bias parameters trained on mean-centered data \citep{bourlard}. So without loss of generality, we assume $X$ is mean centered and consider the following three LAE loss functions for fixed $\lambda > 0$:
\begin{align*}
    \mathcal{L}(W_1,W_2) &= ||X - W_2W_1X||_F^2 \\
    \mathcal{L}_\pi(W_1,W_2) &= \mathcal{L}(W_1,W_2) + \lambda||W_2W_1||_F^2\\
    \mathcal{L}_\sigma(W_1,W_2) &= \mathcal{L}(W_1,W_2) + \lambda(||W_1||_F^2 + ||W_2||_F^2)
\end{align*}
We call these the \textbf{unregularized}, \textbf{product}, and \textbf{sum} losses, respectively.

The product and sum losses mirror the loss functions of a linear \textit{denoising autoencoder} (DAE) \citep{vincent} and linear \textit{contractive autoencoder} (CAE) \citep{rifai} respectively. See Appendix \ref{denoising} for details.


\subsection{Visualizing LAE loss landscapes}
\label{sec1d}

We can visualize these loss functions directly in the case $n = m = k = 1$, as shown in Figure \ref{one dimensional}. In fact, working out the critical points in this scalar case led us to conjecture the general result in Section $\ref{seccrit}$. We invite the reader to enjoy deriving the following results and experimenting with these loss landscapes using our online \href{https://danielkunin.github.io/Regularized-Linear-Autoencoders/}{visualization tool}.

For all three losses, the origin $w_1 = w_2 = 0$ is the unique rank-0 critical point. For $\mathcal{L}$ and $\mathcal{L}_\pi$, the origin is always a saddle point, while for $\mathcal{L}_\sigma$ the origin is either a saddle point or global minimum depending of the value of $\lambda$.

For $\mathcal{L}$, the global minima are rank-1 and consist of the hyperbola\footnote{Identified with the components of $\mathrm{GL}_1(\mathbb{R}) \cong \mathbb{R} \backslash \{0\}$.}
$$w_2 w_1 = 1.$$
For $\mathcal{L}_\pi$, the global minima are rank-1 and consist of this hyperbola shrunk toward the origin as in ridge regression,
$$w_2 w_1 = (1 + \lambda x^{-2})^{-1}.$$ 
For $\mathcal{L}_\sigma$ the critical points depend on the scale of $\lambda$ relative to $x^2$. For $\lambda < x^2$, the origin is a saddle point and the global minima are the two isolated rank-1 critical points\footnote{Identified with the components of $\mathrm{O}_1(\mathbb{R}) \cong \{\pm 1 \}$.} cut out by the equations
$$w_2 w_1 = 1 - \lambda x^{-2}, \quad w_1 = w_2.$$ 
As $\lambda$ increases toward $x^2$, these minima move toward the origin, which remains a saddle point. As $\lambda$ exceeds $x^2$, the origin becomes the unique global minimum. This loss of information was our first hint at the connection to probabilistic PCA formalized in Theorem \ref{thmppca}.

\subsection{Regularization and orthogonality}

Adding $L_2$-regularization to the encoder and decoder separately reduces the symmetries of the loss from $\mathrm{GL}_k(\mathbb{R})$ to the orthogonal group $O_k(\mathbb{R})$. Two additional facts about the relationship between regularization and orthogonality have guided our intuition:

\begin{asparaenum}[(a)]
\item Orthogonal matrices are the determinant $\pm 1$ matrices of minimal Frobenius norm\footnote{Geometrically: the unit-volume parallelotope of minimal total squared side length is the unit hypercube.},
$$\argmin_{A} ||A||^2_F \ \ \ \mathrm{ s.t. } \ \ \det(A)^2 = 1.$$
\item Orthogonal matrices are the inverse matrices of minimum total squared Frobenius norm,
$$\argmin_{A,B} ||A||_F^2 + ||B||_F^2 \ \ \ \mathrm{ s.t. } \ \ AB = I,$$
and in particular $A = B^\intercal$ at all minima.
\end{asparaenum}
Both facts follow from the inequality of arithmetic and geometric means after casting the problems in terms of squared singular values\footnote{The squared Frobenius norm is their sum and the squared determinate is their product.}. While it was not immediately clear to us that the transpose relationship in (b) also holds at the minima of $\mathcal{L}_\sigma$, in fact, all critical points of an L$_2$-regularized linear autoencoder are symmetric:

\begin{theorem}[Transpose Theorem]
\label{orthogonality} All critical points of $\mathcal{L}_\sigma$ satisfy $W_1 = W_2^\intercal.$
\end{theorem}

Our proof uses elementary properties of positive definite matrices as reviewed in Appendix \ref{definite}. Note that without regularization, all critical points are psuedoinverses, $W_1 = W_2^+$, as is clear in the scalar case and derived in Section \ref{seccrit}.

\begin{proof}
Critical points of $\mathcal{L}_\sigma$ satisfy:
\begin{align*}
\frac{\partial \mathcal{L}_\sigma}{\partial W_1} &= 2W_2^\intercal(W_2W_1 - I)XX^\intercal + 2\lambda W_1 = 0,\\
\frac{\partial \mathcal{L}_\sigma}{\partial W_2} &= 2(W_2W_1 - I)XX^\intercal W_1^\intercal + 2\lambda W_2 = 0.
\end{align*}

We first prove that the matrix $$C = (I - W_2 W_1)XX^\intercal$$ is positive semi-definite\footnote{Intuitively, we expect this property so long as $W_2 W_1$ shrinks the principal directions of $X$, so that $I - W_2 W_1$ does as well.}. Rearranging $\frac{\partial \mathcal{L}_\sigma}{\partial W_2} W_2^\intercal$ gives
$$XX^\intercal(W_2 W_1)^\intercal = (W_2 W_1)XX^\intercal (W_2W_1)^\intercal + \lambda W_2 W_2^\intercal.$$
Both terms on the right are positive semi-definite, so their sum on the left is as well and therefore
$$XX^\intercal(W_2W_1)^\intercal \succeq (W_2W_1) XX^\intercal (W_2W_1)^\intercal.$$ Cancelling $(W_2 W_1)^\intercal$ via Lemma \ref{psdprops} gives $C \succeq 0$.

We now show the difference $A = W_1 - W_2^\intercal$ is zero. Rearranging terms using the symmetry of $C$ gives
$$0 = \frac{\partial \mathcal{L}_\sigma}{\partial W_1} - \frac{\partial \mathcal{L}_\sigma}{\partial W_2}^\intercal = 2 A (C + \lambda I).$$
Since $C \succeq 0$ and $\lambda > 0$ imply $C + \lambda I \succ 0$, we conclude from $$A (C + \lambda I) A^T = 0$$
that $A = 0$.
\end{proof}

\section{Bayesian Models}
\label{bayesian models}

In this section, we identify Bayesian counterparts of our three loss functions and derive a novel connection between (regularized) LAEs and (probabilistic) PCA. This connection enables the application of any LAE training method to Bayesian MAP estimation of the corresponding model.

Consider the rank-$k$ (self-)regression model
$$x_i = W_2 W_1 x_i + \varepsilon_i = W x_i + \varepsilon_i$$
where $W_1$ and $W_2$ act through their product $W$ and $\varepsilon_i \sim \mathcal{N}_m(0, 1)$.

\begin{compactitem}
    \item $\mathcal{L}$ is \textit{rank-k regression}. The prior on $W$ is the uniform distribution on $\mathbb{R}^{m \times m}$ restricted to rank-$k$ matrices\footnote{Note that rank-(k-1) matrices are a measure zero subset of rank-$k$ matrices.}.
    
    \item $\mathcal{L}_\pi$ is \textit{rank-k ridge regression}. The prior on $W$ is $\mathcal{N}_{m \times m}(0, \lambda^{-1})$ restricted to rank-$k$ matrices.
    
    \item $\mathcal{L}_\sigma$ is the model with $W_1$ and $W_2^\intercal$ independently drawn from $\mathcal{N}_{k \times m}(0, \lambda^{-1})$.
\end{compactitem}

Theorem \ref{unified statement} shows that the minima of $\mathcal{L}_\sigma$, or equivalently the MAP of the Bayesian model, are such that $W_2$ is the orthogonal projection onto the top $k$ principal directions followed by compression in direction $i$ via multiplication by the factor $(1 - \lambda \sigma_i^{-2})^\frac{1}{2}$ for $\sigma_i^2 > \lambda$ and zero otherwise. Notably, for principal directions with eigenvalues dominated by $\lambda$, all information is lost no matter the number of data points. The same phenomenon occurs for pPCA with respect to eigenvalues dominated by the variance of the noise, $\sigma^2$. Let's consider these Bayesian models side by side, with $W_0 \in \mathbb{R}^{m \times k}$ the parameter of pPCA\footnote{See Chapter 12.2 of \citet{bishop} for background on pPCA.}:

\begin{table}[H]
\centering
\begin{tabular}{@{}cc@{}} \toprule
Bayesian $\mathcal{L}_\sigma$ & pPCA\\ \midrule
$\begin{aligned}[t]
W_1, W_2^\intercal &\sim \mathcal{N}_{k \times m}(0, \lambda^{-1})\\
\varepsilon_i &\sim \mathcal{N}_m(0, 1)\\
x_i &= W_2 W_1 x_i + \varepsilon_i
\end{aligned}$ &
$\begin{aligned}[t]
z_i &\sim \mathcal{N}_k(0, 1)\\
\varepsilon_i &\sim \mathcal{N}_m(0, \sigma^2)\\
x_i &= W_0 z_i + \varepsilon_i\\
\end{aligned}$\\ \bottomrule
\end{tabular}
\end{table}

Comparing the critical points of $\mathcal{L}_\sigma$ in Theorem \ref{unified statement},
\begin{equation}
\label{ppca1}
W_1^T = W_2 = U_\mathcal{I} (I_\ell - \lambda \Sigma^{-2}_\mathcal{I})^{\frac{1}{2}} O^\intercal,
\end{equation}
and pPCA \citep{bishop99},
\begin{equation}
\label{ppca2}
W_0 = U_\mathcal{I} \Sigma_\mathcal{I}(I_\ell - \sigma^2 \Sigma^{-2}_\mathcal{I})^{\frac{1}{2}}O^\intercal,
\end{equation}
where $O \in \mathbb{R}^{k \times \ell}$ has orthonormal columns, we see that $\lambda$ corresponds to $\sigma^2$ (rather than the precision $\sigma^{-2}$) in the sense that principal directions with eigenvalues dominated by either are collapsed to zero. The critical points only differ in the factor by which the the remaining principal directions are shrunk. More precisely:

\begin{theorem}[pPCA Theorem]
\label{thmppca}
With $\sigma^2 = \lambda$, the critical points of $$\mathcal{L}^0_\sigma(W_1, W_2) = \mathcal{L}_\sigma(W_1(XX^\intercal)^{-\frac{1}{2}}, (XX^\intercal)^{-\frac{1}{2}} W_2)$$ coincide with the critical points of pPCA.
\end{theorem}
\begin{proof}
Multiplying the expression for $W_2$ in \eqref{ppca1} on the left by $(XX^\intercal)^{\frac{1}{2}}$ gives the expression for $W_0$ in \eqref{ppca2}.
\end{proof}
Interestingly, the generative model for $\mathcal{L}^0_\sigma$ differs from that of pPCA. For example, in the scalar case $\mathcal{L}^0_\sigma(w_1, w_2)$ is
$$x^2 \left(1-x^{-2} w_2w_1\right)^2 + \lambda x^{-2}(w_1^2 + w_2^2)$$
whereas the negative log likelihood of pPCA is
$$\frac{1}{2}\left(\ln(2\pi) + \ln(w_0^2+\sigma^2) + x^2 (w_0^2 + \sigma^2)^{-1}\right).$$

\section{Loss Landscapes}
\label{seccrit}

Having contextualized LAE models in a Bayesian framework, we now turn to understanding their loss landscapes. Symmetries such as \eqref{symg} exist because the model is expressed in an ``overparameterized" coordinate form rooted in classical linear algebra. This results in ``flat" critical manifolds rather than a finite number of critical points. In Section \ref{secgrass}, we remove all symmetries by expressing the loss geometrically over a topological domain. This results in $\binom{m}{k}$ critical points, and in particular a unique minimum. This intuition will pay off in Sections \ref{alltogether} and \ref{curvature}, where we fully characterize the critical manifolds and local curvatures of all three LAE loss landscapes.

\subsection{Critical points}
\label{secgrass}

We now consider reconstruction loss over the domain of $k$-dimensional planes through the origin in $\mathbb{R}^m$. This space has the structure of a $k(m-k)$-dimensional smooth, compact manifold called the Grassmannian of $k$-planes in $\mathbb{R}^m$ and denoted $\mathrm{Gr}_k(\mathbb{R}^m)$ \citep{hatcher}. We'll build intuition with a few simple examples.

\begin{figure}[tb]
\vskip 0.2in
\begin{center}
\begin{subfigure}{\columnwidth}
  \centerline{
  \includegraphics[width=\linewidth]{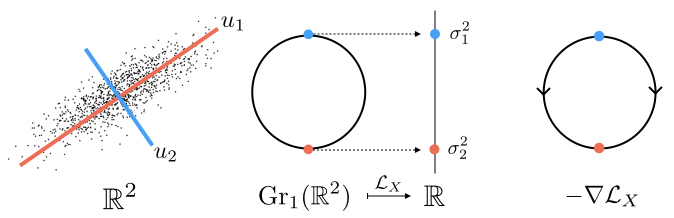}
  }
  \caption{Lines in the plane.}
\end{subfigure}
\begin{subfigure}{\columnwidth}
  \centerline{
  \includegraphics[width=\linewidth]{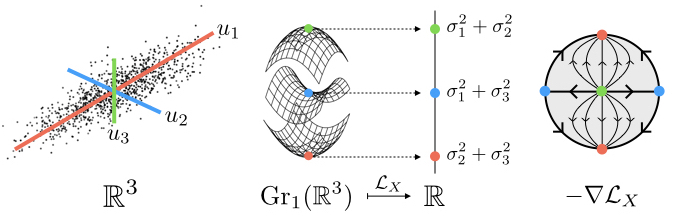}
  }
  \caption{Lines in space.}
\end{subfigure}
\caption{\textit{Left:} Principal directions of a point cloud $X$. \textit{Middle:} $\mathcal{L}_X$ as height function on the manifold of lines through the origin. \textit{Right:} Negative gradient flow of $\mathcal{L}_X$.}
\label{grassman}
\end{center}
\vskip -0.2in
\end{figure}

\begin{compactitem}
\item $\mathrm{Gr}_1(\mathbb{R}^2)$ is the space of lines through the origin in the plane, which may be smoothly parameterized by a counterclockwise angle of rotation of the $x$-axis modulo the half turn that maps a line to itself.

\item $\mathrm{Gr}_1(\mathbb{R}^3)$ is the space of lines through the origin in 3-space, also known as the real projective plane. We can visualize $\mathrm{Gr}_1(\mathbb{R}^3)$ as the northern hemisphere of the 2-sphere with equator glued by the antipodal map.

\item $\mathrm{Gr}_2(\mathbb{R}^3)$ is identified with $\mathrm{Gr}_1(\mathbb{R}^3)$ by mapping a plane to its 1-dimensional orthogonal complement.
\end{compactitem}

A point cloud $X$ in $\mathbb{R}^m$ determines a smooth function $$\mathcal{L}_X: \mathrm{Gr}_k(\mathbb{R}^m) \to \mathbb{R}$$ whose value on a $k$-plane is the sum of square distances from the points to the plane. Figure \ref{grassman} depicts $\mathcal{L}_X$ as a height function for $\mathrm{Gr}_1(\mathbb{R}^2)$ and $\mathrm{Gr}_1(\mathbb{R}^3)$. Note that the min and max in (a) are the principal directions $u_1$ and $u_2$, while the min, saddle, and max in (b) are the principal directions $u_1$, $u_2$, and $u_3$. At right, we depict the negative gradient flow $-\nabla \mathcal{L}_X$ on each space, with (b) represented as a disk with glued boundary. In (a), $u_1$ may descend to $u_2$ by rotating clockwise or counterclockwise\footnote{Formally, we mean there are two gradient trajectories (modulo time translation) that converge to $u_1$ and $u_2$ in each time direction asymptotically, namely the left and right halves of the circle.}. In (b), $u_1$ may descend to $u_2$ or $u_3$ by rotating in either of two directions in the plane they span, and $u_2$ may similarly descend to $u_3$.

The following theorem requires our assumption that the singular values of $X$ are distinct. 

\begin{theorem}[Grassmannian Theorem]
\label{perfect}
$\mathcal{L}_X$ is a smooth function with $\binom{m}{k}$ critical points given by all rank-$k$ principal subspaces. In local coordinates near the critical point with principal directions $i_1 < \ldots < i_k$, $\mathcal{L}_X$ takes the form of a standard non-degenerate saddle with
\begin{equation}
\label{morseindex}
d_\mathcal{I} = \sum_{j=1}^k (i_j - j).
\end{equation}
descending directions.
\end{theorem}
The latter formula counts the total number of pairs $i < j$ with $i \in \mathcal{I}$ and $j \notin \mathcal{I}$. These correspond to directions to flow along $-\nabla\mathcal{L}_X$ by rotating one principal direction $u_i$ to another $u_j$ of higher eigenvalue, fixing the rest.

In Appendix \ref{morse}, we prove a stronger form of the Grassmannian Theorem by combining Theorem \ref{unified statement}, a commutative diagram \eqref{commutegrass} relating $\mathcal{L}_X$ and $\mathcal{L}$, and techniques from algebraic topology.

\subsection{Critical manifolds}
\label{alltogether}

Translating the Grassmannian Theorem back to the coordinate representation $\mathbb{R}^{k \times m} \times \mathbb{R}^{m \times k}$ introduces two additional phenomena we saw in the scalar case in Section \ref{sec1d}. 

\begin{compactitem}
\item Each critical point on $\mathrm{Gr}_k(\mathbb{R}^m)$ corresponds to a manifold $\mathrm{GL}_k(\mathbb{R})$ or $\mathrm{O}_k(\mathbb{R})$ of rank-$k$ critical points.

\item Critical manifolds appear with rank less than $k$. In particular, $(0, 0)$ is a critical point for all three losses.
\end{compactitem}

Now let's now combine our topological and scalar intuition to understand the the loss landscapes of LAEs in all dimensions and for all three losses.

Theorem \ref{unified statement} requires our assumption\footnote{For the sum loss, we also assume $\lambda$ is distinct from all $\sigma_i^2$.} that $X$ has distinct singular values $\sigma_1 > \dots > \sigma_m > 0$. Let $u_1, \dots, u_m$ denote the corresponding left singular vectors (or principal directions) of $X$. For an index set $\mathcal{I} \subset \{1, \ldots, m \}$ we define:
\begin{compactitem}
    \item[$\bullet$] $\ell = |\mathcal{I}|$ and increasing indices $i_1 < \dots < i_\ell$,
    \item[$\bullet$] $\Sigma_\mathcal{I} = \mathrm{diag}(\sigma_{i_1}, \dots, \sigma_{i_\ell})\in \mathbb{R}^{\ell \times \ell}$,
    \item[$\bullet$] $U_\mathcal{I} \in \mathbb{R}^{m \times \ell}$ consisting of columns $i_1, \ldots, i_\ell$ of $U$,
    \item $F_\mathcal{I}$, the open submanifold of $\mathbb{R}^{k \times \ell}$ whose points are matrices $G$ with independent columns\footnote{Also known as the manifold of $l$-frames in $\mathbb{R}^k$.},
    \item $V_\mathcal{I}$, the closed submanifold of $\mathbb{R}^{k \times \ell}$ whose points are matrices $O$ with orthonormal columns\footnote{Also known as the Stiefel manifold $V_l(\mathbb{R}^k)$.}.
\end{compactitem}

\begin{theorem}[Landscape Theorem]
\label{unified statement}
For each loss, the critical points form a smooth submanifold  of $\mathbb{R}^{k \times m} \times \mathbb{R}^{m \times k}$.

For $\mathcal{L}$ and $\mathcal{L}_\pi$, this submanifold is diffeomorphic to the disjoint union of $F_\mathcal{I}$ over all $\mathcal{I} \subset \{1, \ldots, m \}$ of size at most $k$.

For $\mathcal{L}_\sigma$, this submanifold is diffeomorphic to the disjoint union of $V_\mathcal{I}$ over all $\mathcal{I} \subset \{1, \ldots, m_0 \}$ of size at most $k$, where $m_0$ is the largest index such that $\sigma_{m_0}^2 > \lambda$.

These diffeomorphisms map $G \in F_\mathcal{I}$ or $O \in V_\mathcal{I}$ to a critical point $(W_1, W_2)$ as follows:

\begin{table}[H]
\hspace*{-0.25cm}
\centering
\begin{tabular}{@{}lcc@{}} \toprule
    & $W_2$  & $W_1$ \\ \midrule
    $\mathcal{L}$ & $U_\mathcal{I} G^+$ & $G U_\mathcal{I}^\intercal$ \\
    $\mathcal{L}_\pi$ & $U_\mathcal{I} (I_\ell + \lambda \Sigma^{-2}_\mathcal{I})^{-\frac{1}{2}}G^+$ & $G (I_\ell + \lambda \Sigma^{-2}_\mathcal{I})^{-\frac{1}{2}} U_\mathcal{I}^\intercal$ \\
    $\mathcal{L}_\sigma$ & $U_\mathcal{I} (I_\ell - \lambda \Sigma^{-2}_\mathcal{I})^\frac{1}{2} O^\intercal$ & $O (I_\ell - \lambda \Sigma^{-2}_\mathcal{I})^\frac{1}{2} U_\mathcal{I}^\intercal$\\ \bottomrule
\end{tabular}
\end{table}
\end{theorem}

The proof of the Landscape Theorem \ref{unified statement} follows quickly from the Transpose Theorem \ref{orthogonality} and Proposition \ref{diagonal critical points}.

\begin{prop}
\label{diagonal critical points}
Let $D, S \in \mathbb{R}^{m \times m}$ be diagonal matrices such that $S$ is invertible and the diagonal of $D^2 S^{-2}D^2$ has distinct non-zero elements. Then the critical points of
$$\mathcal{L}_*(Q_1, Q_2) = \tr(Q_2 Q_1 S^2Q_1^\intercal Q_2^\intercal - 2Q_2Q_1D^2)$$
are smoothly parameterized as the disjoint union of $F_\mathcal{I}$ over all $\mathcal{I} \subset \{1, \ldots, m \}$ of size at most $k$. The diffeomorphism maps $G \in F_\mathcal{I}$ to a critical point $(Q_1, Q_2)$ as follows\footnote{Here $D_\mathcal{I}$ and $S_\mathcal{I}$ are defined like $\Sigma_\mathcal{I}$. $I_\mathcal{I}$ is defined like $U_\mathcal{I}$.}:
\begin{table}[H]
\centering
\begin{tabular}{@{}ccc@{}} \toprule
    & $Q_2$  & $Q_1$ \\ \midrule
    $\mathcal{L}_*$ & $D_\mathcal{I} S_\mathcal{I}^{-1} I_\mathcal{I} G^+$ & $G I_\mathcal{I}^\intercal D_\mathcal{I} S_\mathcal{I}^{-1}$\\ \bottomrule
\end{tabular}
\end{table}
\end{prop}

\begin{proof}[Proof of Theorem \ref{unified statement}]

Given the singular value decomposition $X = U\Sigma V^\intercal$, let $Q_1 = W_1U$ and $Q_2 = U^\intercal W_2$. By invariance of the Frobenius norm under the smooth action of the orthogonal group, we may instead parameterize the critical points of the following loss functions and then pull the result back to $W_1 = Q_1U^\intercal$ and $W_2 = UQ_2$:
\begin{align*}
\mathcal{L}(Q_1, Q_2) &= ||\Sigma - Q_2Q_1\Sigma||_F^2\\
\mathcal{L}_\pi(Q_1, Q_2) &= \mathcal{L}(Q_1, Q_2) + \lambda ||Q_2Q_1||_F^2\\
\mathcal{L}_\sigma(Q_1, Q_2) &= \mathcal{L}(Q_1, Q_2) + \lambda(||Q_1||_F^2 + ||Q_2||_F^2)
\end{align*}

$\mathcal{L}$ expands to
$$\tr(Q_2 Q_1\Sigma^2 Q_1^\intercal Q_2^\intercal -2Q_2Q_1\Sigma^2 + \Sigma^2).$$
By Proposition \ref{diagonal critical points} with $S = \Sigma$ and $D = \Sigma$, the critical points have the form
\begin{table}[H]
\centering
\begin{tabular}{@{}ccc@{}} \toprule
    & $Q_2$  & $Q_1$ \\ \midrule
    $\mathcal{L}$ & $I_\mathcal{I}G^{+}$ & $GI_\mathcal{I}^\intercal$\\ \bottomrule
\end{tabular}
\end{table}

$\mathcal{L}_\pi$ expands to
$$\tr(Q_2 Q_1 (\Sigma^2 + \lambda I) Q_1^\intercal Q_2^\intercal -2Q_2Q_1\Sigma^2 + \Sigma^2).$$
By Proposition \ref{diagonal critical points} with $S = (\Sigma^2 + \lambda I)^{\frac{1}{2}}$ and $D = \Sigma$, the critical points have the form
\begin{table}[H]
\centering
\begin{tabular}{@{}ccc@{}} \toprule
    & $Q_2$  & $Q_1$ \\ \midrule
    $\mathcal{L}_\pi$ & $I_\mathcal{I}(I_\ell + \lambda \Sigma^{-2}_\mathcal{I})^{-\frac{1}{2}}G^{+}$ & $G(I_\ell + \lambda \Sigma^{-2}_\mathcal{I})^{-\frac{1}{2}}I_\mathcal{I}^\intercal$\\ \bottomrule
\end{tabular}
\end{table}
By Lemma \ref{sum of fro} with $A = Q_2$ and $B = Q_1$, $\mathcal{L}_\sigma$ expands to the sum of two functions:
\begin{align*}
\mathcal{L}_1(Q_1, Q_2) &= \tr(Q_2 Q_1 \Sigma^2 Q_1^\intercal Q_2^\intercal -2Q_2Q_1(\Sigma^2 - \lambda I) + \Sigma^2)\\
\mathcal{L}_2(Q_1, Q_2) &= \lambda ||Q_1 - Q_2^\intercal||_F^2.
\end{align*}
So at a critical point, $\nabla \mathcal{L}_1(Q_1, Q_2) = -\nabla \mathcal{L}_2(Q_1, Q_2)$ and $Q_1 = Q_2^\intercal$ by the Transpose Theorem \ref{orthogonality}. The latter also implies $\nabla \mathcal{L}_2(Q_1, Q_2) = 0$. So the critical points of $\mathcal{L}_\sigma$ coincide with the critical points of $\mathcal{L}_1$ such that $Q_1 = Q_2^\intercal$.

By Proposition \ref{diagonal critical points} with $S = \Sigma$ and $D = (\Sigma^2 - \lambda I)^{\frac{1}{2}}$, these critical points have the form
\begin{table}[H]
\centering
\begin{tabular}{@{}ccc@{}} \toprule
    & $Q_2$  & $Q_1$ \\ \midrule
    $\mathcal{L}_\sigma$ & $I_\mathcal{I}(I_\ell - \lambda \Sigma^{-2}_\mathcal{I})^{\frac{1}{2}}O^\intercal$ & $O(I_\ell - \lambda \Sigma^{-2}_\mathcal{I})^{\frac{1}{2}}I_\mathcal{I}^\intercal$\\ \bottomrule
\end{tabular}
\end{table}
In particular, real solutions do not exist for $\sigma_i^2 < \lambda$.
\end{proof}

\subsection{Local curvature}
\label{curvature}

By Theorem \ref{unified statement}, the critical landscape is a disjoint union of smooth manifolds, each at some height. We now prove that the Hessian is non-degenerate in the normal directions to the critical landscape. As discussed in Appendix \ref{morse}, such functions are called \textit{Morse-Bott} and studied extensively in differential and algebraic topology.

\begin{theorem}[Curvature Theorem]
\label{morse-bott}
In local coordinates near any point on the critical manifold indexed by $\mathcal{I}$, all three losses take the form of a standard degenerate saddle with $d_\mathcal{I} + (k-\ell)(m-\ell)$ descending directions.
\begin{compactitem}
\item $\mathcal{L}$ and $\mathcal{L}_\pi$ have $k\ell$ flat directions.
\item  $\mathcal{L}_\sigma$ has $k\ell - \binom{\ell + 1}{2}$ flat directions.
\end{compactitem}
The remaining directions are ascending.
\end{theorem}
\begin{proof}
In addition to the descending directions of $\mathcal{L}_X$, there are $(k-\ell)(m-\ell)$ more that correspond to scaling one of $k-\ell$ remaining slots in $G$ or $O$ toward one of $m-\ell$ available principal directions\footnote{In Figure \ref{one dimensional}c, these two gradient trajectories descend from the yellow saddle at $0$ to the two red minima at $\pm u_1$.}. For $\mathcal{L}$ and $\mathcal{L}_\pi$, the ascending directions are the $k(m-k) - d_\mathcal{I}$ ascending directions of $\mathcal{L}_X$; for $\mathcal{L}_\sigma$, an additional $\binom{\ell + 1}{2}$ ascending directions preserve the reconstruction term while increasing the regularization term by decreasing orthogonality. The remaining (flat) directions are tangent to the critical manifold, itself diffeomorphic to the manifold of (orthonormal) $l$-frames in $\mathbb{R}^k$.
\end{proof}

\section{Empirical Illustration}
\label{secemp}

In this section, we illustrate the Landscape Theorem \ref{unified statement} by training an LAE on synthetic data and visualizing properties of the learned weight matrices. See Appendix \ref{mnist experiments} for experiments on real data. Corollary \ref{morse-bott} implies that gradient descent and its extensions will reach a global minimum, regardless of the initialization, if trained for a sufficient number of epochs with a small enough learning rate \citep{zhu}.

\subsection{Synthetic data}
\label{synetheic experiments}

\begin{figure}[tb]
\begin{center}
\centerline{\includegraphics[width=\columnwidth]{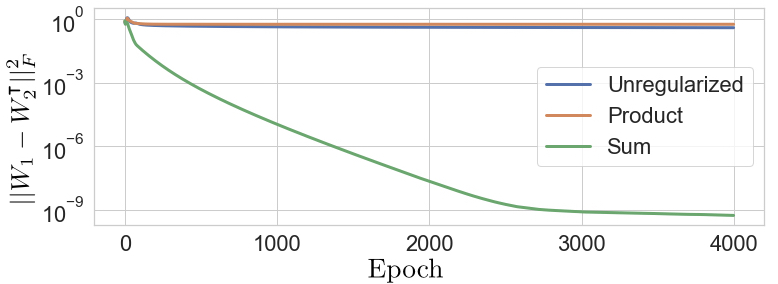}}
\caption{Distance between $W_1$ and $W_2^\intercal$ during training.}
\label{difference}
\end{center}
\vskip -0.2in
\end{figure}

\begin{figure}[tb]
\begin{center}
\begin{subfigure}{0.32\linewidth}
  \includegraphics[width=\columnwidth]{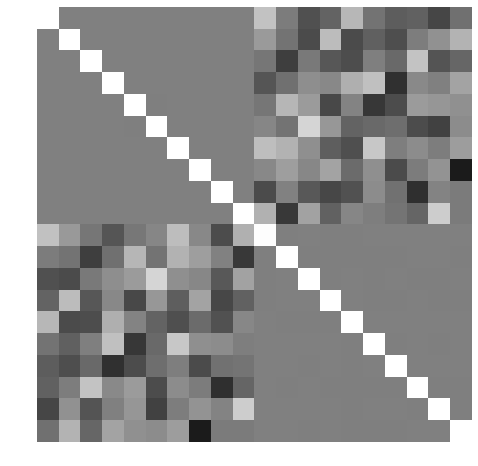}
  \caption{Unregularized}
\end{subfigure}
\begin{subfigure}{0.32\linewidth}
  \includegraphics[width=\columnwidth]{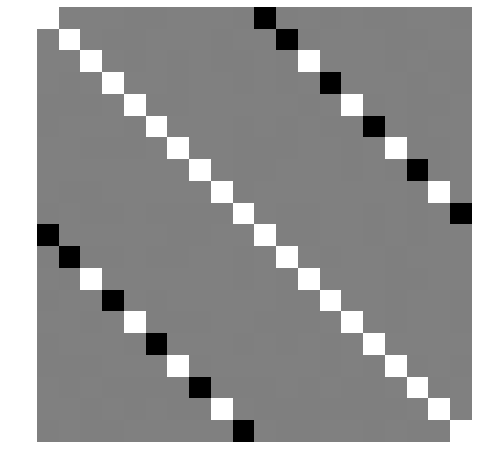}
  \caption{Product}
\end{subfigure}
\begin{subfigure}{0.32\linewidth}
  \includegraphics[width=\columnwidth]{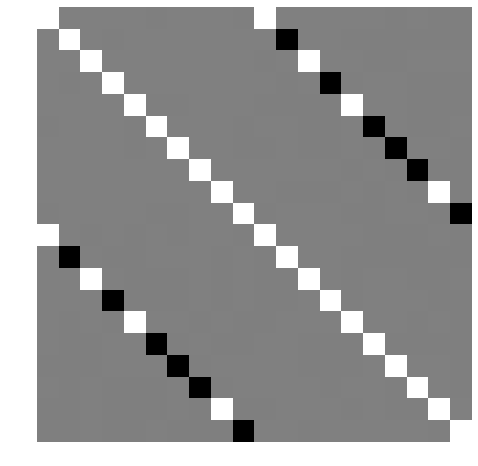}
  \caption{Sum}
\end{subfigure}
\caption{Heat map of the matrix $\begin{bmatrix}U & V_*\end{bmatrix}^\intercal \begin{bmatrix}U & U_*\end{bmatrix}$. Black and white correspond to $-1$ and $1$, respectively.}
\label{singular vectors}
\end{center}
\vskip -0.2in
\end{figure}

\begin{figure}[tb]
\begin{center}
\begin{subfigure}{0.49\linewidth}
  \centerline{
  \includegraphics[width=\columnwidth]{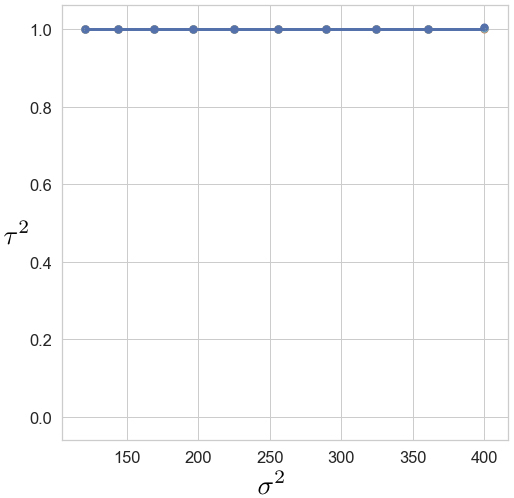}
  }
  \caption{Unregularized}
\end{subfigure}
\begin{subfigure}{0.49\linewidth}
  \centerline{
  \includegraphics[width=\columnwidth]{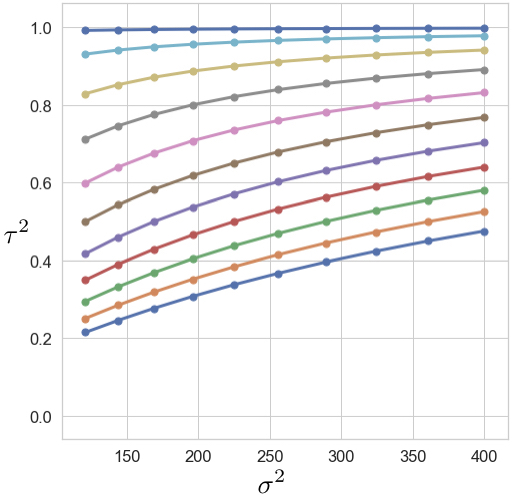}
  }
  \caption{Product}
\end{subfigure}
\begin{subfigure}{0.74\linewidth}
  \centerline{
  \includegraphics[width=\columnwidth]{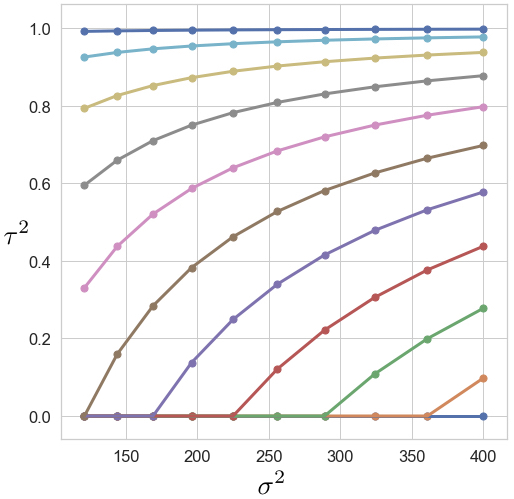}
  }
  \caption{Sum}
\end{subfigure}
\begin{subfigure}{0.25\linewidth}
  \centerline{
  \includegraphics[width=0.8\columnwidth]{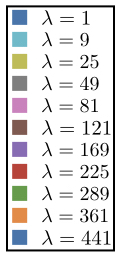}
  }
    \caption*{}
\end{subfigure}
\caption{Illustration of the relationship between the eigenvalues of the weight matrix ($\tau^2$) and data matrix ($\sigma^2$) for various values of $\lambda$. Points are the empirical and lines are theoretical.}
\label{singular values}
\end{center}
\vskip -0.2in
\end{figure}

In the following experiments, we set $k = \lambda = 10$ and fix a data set $X \in \mathbb{R}^{20\times20}$ with singular values $\sigma_i = i$ and random left and right singular vectors under Haar measure. We train the LAE for each loss using the Adam optimizer for $4000$ epochs with random normal initialization, full batch, and learning rate $0.05$.

Figure \ref{difference} tracks the squared distance between $W_1$ and $W_2^\intercal$ during training. Indeed, only the the sum loss pulls the encoder and (transposed) decoder together as claimed in the Transpose Theorem \ref{orthogonality}.

Let $W_*$ be the product $W_2 W_1$ after training and fix the singular value decompositions
$$X = U \Sigma V^\intercal, \quad W_* = U_* \Sigma_* V_*^\intercal.$$
For each loss, the heat map of
$$\begin{bmatrix}
U^\intercal U & U^\intercal U_* \\
V_*^\intercal U & V_*^\intercal U_*
\end{bmatrix}$$
in Figure \ref{singular vectors} is consistent with $W_*$ approximating a global minimum defined by the Landscape Theorem. Namely, the lower right quadrant shows $U_* \approx V_*$ for each loss and the upper right and lower left quadrants show $U \approx U_*$ and $U \approx V_*$ up to column sign for the product and sum losses, \textit{but not for the unregularized loss}. That is, for the product and sum losses, the left singular vectors of $X$ are obtained as the right and left singular vectors of $W_*$.

The Landscape Theorem also gives explicit formulae for the eigenvalues of $W_*$ at convergence. Letting $\sigma_i^2$ and $\tau_i^2$ be the $i^{th}$ largest eigenvalues of $XX^\intercal$ and $W_*$, respectively, in Figure \ref{singular values} we plot the points
$\left(\sigma_i^2, \tau_i^2\right)$
for many values of $\lambda$. We superimpose a curve for each value of $\lambda$ defined by the theoretical relationship between $\sigma_i^2$ and $\tau_i^2$ in the Landscape Theorem. The (literal) alignment of theory and practice is visually perfect.

\section{Implications}
\label{secdisc}

\subsection{PCA algorithms}

The Landscape Theorem for $\mathcal{L}_\sigma$ implies those principal directions of $X$ with eigenvalues greater than $\lambda$ coincide with the top left singular vectors of the trained decoder. Hence we can perform PCA by training a regularized LAE and computing SVD of the decoder. For example, full-batch gradient descent corresponds to the following algorithm, which is guaranteed to converge to a (global) minimum for sufficiently small learning rate by the Curvature Theorem.

\begin{algorithm}[H]
  \caption{LAE--PCA}
  \label{lae-pca}
    \begin{algorithmic}
      \STATE {\bfseries input} $X \in \mathbb{R}^{m \times n}$; $k \leq m$; $\lambda, \alpha > 0$
      \STATE {\bfseries initialize} $W_1, W_2^\intercal \in \mathbb{R}^{k \times m}$
      \STATE {\bfseries while not converged}
      \STATE \quad $W_1 \mathrel{-}= \alpha \left( W_2^\intercal(W_2W_1 - I)XX^\intercal + \lambda W_1\right)$
      \STATE \quad $W_2 \mathrel{-}= \alpha \left( (W_2W_1 - I)XX^\intercal W_1^\intercal + \lambda W_2\right)$
      \STATE $U, \Sigma, \underline{\hspace{2mm}} = \mathrm{SVD}(W_2)$
\STATE {\bfseries return} $U, \ \lambda(I - \Sigma^2)^{-1}$
    \end{algorithmic}
\end{algorithm}

Note the SVD step is trivial because the decoder is only $m \times k$ dimensional with $k \ll n$. Hence optimizing this formulation of PCA reduces to optimizing the training of a regularized LAE. We have so far explored several simple ideas for optimizing performance\footnote{NumPy implementations and benchmarks of these ideas are available on \href{https://github.com/danielkunin/Regularized-Linear-Autoencoders}{GitHub}.}:

\begin{compactitem}
    \item \textit{First order.} The Transpose Theorem suggests ``tying" $W_1 = W_2^\intercal$ \textit{a priori}. In fact, for $\lambda=0$ and $k=1$, the gradient update for $\mathcal{L}_\sigma$ with tied weights is equivalent to Oja's rule for how neurons in the brain adapt synaptic strength \citep{Oja}. See Appendix \ref{oja proof} for a derivation.
    \item \textit{Second order.} The loss is convex in each parameter fixing the other, so $W_1$ and $W_2$ may be updated exactly per iteration by solving a system of $m$ or $k$ linear equations respectively.
\end{compactitem}

While many other methods have been proposed for recovering principal components with neural networks, they all require specialized algorithms for iteratively updating weights, similar to classical numerical SVD approaches \citet{Warmuth,Feng1,Feng2}. By contrast, reducing PCA to a small SVD problem by training a regularized LAE more closely parallels randomized SVD \citep{Halko}. We hope others will join us in investigating how far one can push the performance of \text{LAE-PCA}.

\subsection{Neural alignment}

In gradient descent via backpropagation, the weight matrix before a layer is updated by an error signal that is propagated by the transposed weight matrix after the layer. In the brain, there is no known physical mechanism to reuse feedforward connections for feedback or to enforce weight symmetry between distinct forward and backward connections. The latter issue is known as the \textit{weight transport problem} \citep{Grossberg} and regarded as central to the implausibility of backpropagation in the brain. Recently \citet{feedback} showed that forward weights can sufficiently align to fixed, random feedback weights to support learning in shallow networks, but \citet{Bartunov} demonstrated that this \textit{feedback alignment} and other biologically-plausible architectures break down for deep networks.

We are now investigating two approaches to weight transport inspired by the results herein. First, the Transpose Theorem \ref{orthogonality} suggests that weight symmetry emerges dynamically when feedback weights are updated to optimize forward-backward reconstruction between consecutive layers with weight decay. This \textit{information alignment} algorithm balances task prediction, information transmission, and energy efficiency. Second, Lemma \ref{sum of fro} expresses weight symmetry as balancing weight decay and self-amplification, lending greater biological plausibility to \textit{symmetric alignment}. We have verified that both information and (not surprisingly) symmetric alignment indeed align weights and are competitive with backprop when trained on MNIST and CIFAR10. In parallel, \citet{lillicrap} has independently shown that similar forms of local dynamic alignment scale to ImageNet.

\subsection{Morse homology}

In Appendix \ref{morse} we expand from the algebraic topology of learning PCA to that of learning by gradient descent in general. For example, we explain how Morse homology provides a principled foundation for the empirical observation of low-lying valley passages between minima in \citet{Garipov}. We are hopeful that this perspective will yield insights for efficient training and improved robustness and interpretation of consensus representations and ensemble predictions, for non-convex models arising in matrix factorization and deep learning.

\section{Conclusion}

In 1989, \citet{baldi1989} characterized the loss landscape of an LAE. In 2018, \citet{zhou} characterized the loss landscape of an autoencoder with ReLU activations on a single hidden layer. This paper fills out and ties together the rich space of research on linear networks over the last forty years by deriving from first principles a unified characterization of the loss landscapes of LAEs with and without regularization, while introducing a rigorous topological lens. By considering a simple but fundamental model with respect to regularization, orthogonality, and Morse homology, this work also suggests new principles and algorithms for learning.

\clearpage
\appendix
\begin{section}{Deferred Proofs}
\begin{proof}[\textbf{Proof of Proposition \ref{diagonal critical points}}]
Setting the gradient to zero:
\begin{align}
\label{grad1}
\frac{\partial \mathcal{L}_*}{\partial Q_1} = 0 &\implies Q_2^\intercal Q_2Q_1S^2 = Q_2^\intercal D^2\\
\label{grad2}
\frac{\partial \mathcal{L}_*}{\partial Q_2} =0 &\implies  Q_2Q_1S^2Q_1^\intercal = D^2Q_1^\intercal
\end{align}
Let $Q = Q_2Q_1$. Multiplying \eqref{grad1} on the right by $D^{-2}$ and on the left by $D^{-2}S^2Q_1^\intercal$ gives
$$D^{-2}S^2Q^\intercal QS^2D^{-2} = D^{-2}S^2Q^\intercal,$$
which implies $D^{-2}S^2Q^\intercal$ is symmetric and idempotent. Multiplying \eqref{grad2} on the right by $Q_2^\intercal$ gives
$$Q S^2 Q^\intercal = D^2Q^\intercal,$$
which can be rewritten as
$$Q S^2 Q^\intercal = \left(D^2S^{-2}D^2\right)\left(D^{-2}S^2Q^\intercal\right).$$
Since the left-hand side is symmetric, $D^{-2}S^2Q^\intercal$ is diagonal and idempotent by Lemma \ref{commute diagonal} with $A = D^{-2}S^2Q^\intercal$ and $B = D^2S^{-2}D^2$. Lemma \ref{diagonal idempotent} with the same $A$ implies there exists an index set $\mathcal{I}$ of size $\ell$ with $0 \leq \ell \leq k$ such that
$$D^{-2}S^2Q^\intercal = I_\mathcal{I} I_\mathcal{I}^\intercal$$
and hence
\begin{equation}
\label{q_eqn}
Q = I_\mathcal{I}I_\mathcal{I}^\intercal D^2S^{-2} = I_\mathcal{I}D_\mathcal{I}^2S_\mathcal{I}^{-2}I_\mathcal{I}^\intercal.
\end{equation}
Consider the smooth map $(Q_1, Q_2) \mapsto (\mathcal{I}, G)$ with 
$$G = Q_1S_\mathcal{I}D_\mathcal{I}^{-1}I_\mathcal{I}$$
from the critical submanifold of $\mathbb{R}^{k \times m} \times \mathbb{R}^{m \times k}$ to the manifold of pairs $(\mathcal{I}, G)$ with $G$ full-rank.
Note $$G^+ = I_\mathcal{I}^\intercal S_\mathcal{I}D_\mathcal{I}^{-1}Q_2$$ by \eqref{q_eqn}. Commuting diagonal matrices to rearrange terms in \eqref{grad1} and \eqref{grad2}, we obtain a smooth inverse map from pairs $(\mathcal{I}, G)$ to critical points:
\begin{align*}
Q_1 &= Q_1S^2Q^\intercal D^{-2} = Q_1 I_\mathcal{I} I_\mathcal{I}^\intercal = G I_\mathcal{I}^\intercal D_\mathcal{I} S_\mathcal{I}^{-1},\\
Q_2 &= D^{-2}S^2Q^\intercal Q_2 = I_\mathcal{I} I_\mathcal{I}^\intercal Q_2 = D_\mathcal{I} S_\mathcal{I}^{-1} I_\mathcal{I} G^+.
\end{align*}
\end{proof}

\begin{proof}[\textbf{Equating bias parameters and mean centering.}]
\label{bias and mean centering}

Consider the loss function
$$\mathcal{L}_\beta(W_1,W_2,b_1,b_2) = ||X - W_2(W_1X + b_1e_n^\intercal) + b_2e_n^\intercal||_F^2,$$
where $b_1 \in \mathbb{R}^k$ and $b_2 \in \mathbb{R}^m$ are bias vectors and $e_n \in \mathbb{R}^n$ is the vector of ones. With $b = W_2b_1 + b_2$, $\mathcal{L}_b$ becomes
\begin{equation}
\label{eqb}
||X - W_2W_1X - be_n^\intercal||_F^2.
\end{equation}
At a critical point,
$$\frac{\partial\mathcal{L}_\beta}{\partial b} = 2(X - W_2W_1X - be_n^\intercal)e_n = 0,$$
which implies
$$b = \frac{1}{n}Xe_n - W_2W_1\frac{1}{n}Xe_n.$$
Substituting into \eqref{eqb}, $\mathcal{L}_b$ reduces to 
$$||\bar X - W_2W_1\bar X||_F^2$$
with $\bar X = X - \frac{1}{n}Xe_ne_n^\intercal$. Thus, at the optimal bias parameters, $\mathcal{L}_\beta$ with $X$ is equivalent to $\mathcal{L}$ with $X$ mean-centered.
\end{proof}

\begin{lemma}
\label{sum of fro}
Let $A \in \mathbb{R}^{m \times k}$ and $B \in \mathbb{R}^{k \times m}$, then
$$||A||_F^2 + ||B||_F^2 = ||A - B^\intercal||_F^2 + 2 tr(AB)$$
\end{lemma}
\begin{proof}
\begin{align*}
||A - B^\intercal||_F^2 &= \tr((A - B^\intercal)^\intercal (A - B^\intercal))\\
&= \tr(A^\intercal A - A^\intercal B^\intercal - BA + B^\intercal B)\\
&= ||A||_F^2 + ||B||_F^2 - 2tr(AB)
\end{align*}
\end{proof}

\begin{lemma}
\label{commute diagonal}
Let $A, B \in \mathbb{R}^{m \times m}$ with $B$ diagonal with distinct diagonal elements. If $AB = BA$ then $A$ is diagonal.
\end{lemma}
\begin{proof}
Expand the difference of $(i,j)$ elements:
\begin{align*}
    (AB)_{ij} - (BA)_{ij} &= a_{ij}b_{jj} - b_{ii}a_{ij}\\
    &= a_{ij}(b_{jj} - b_{ii}) = 0.
\end{align*}
So for $i \neq j$, $b_{ii} \neq b_{jj}$ implies $a_{ij} = 0$.
\end{proof}

\begin{lemma}
\label{diagonal idempotent}
If $A \in \mathbb{R}^{m \times m}$ is diagonal and idempotent then $a_{ii} \in \{0,1\}$.
\end{lemma}
\begin{proof}
$0 = (AA - A)_{ii} = a_{ii}^2 - a_{ii} = a_{ii}(a_{ii} - 1).$
\end{proof}

\begin{proof}[\textbf{Relationship between Oja's rule and LAE-PCA}]
\label{oja proof}
The update step used in Oja's rule is
$$\nabla w = \alpha (xy - wy^2),$$
where $\alpha$ is a fixed learning rate, $x, w \in \mathbb{R}^m$ and $y = x^\intercal w$. Substituting $y$ into this update and factoring out $xx^\intercal w$ on the right gives,
$$\nabla w = \alpha (1 - ww^\intercal )xx^\intercal w,$$
which is the (negative) gradient for an unregularized LAE with tied weights in the $k = 1$ case.
\end{proof}

\end{section}

\begin{section}{Positive (semi-)definite matrices}
\label{definite}

We review positive definite and semi-definite matrices as needed to prove the Transpose Theorem (\ref{orthogonality}).

\begin{definition}
A real, symmetric matrix $A$ is positive semi-definite, denoted $A \succeq 0$, if $x^\intercal A x \ge 0$ for all vectors $x$. $A$ is positive definite, denoted $A \succ 0$, if the inequality is strict.

The Loewner partial ordering of positive semi-definite matrices defines $A \succeq B$ if $A - B \succeq 0$.
\end{definition}

\begin{lemma}
\label{psdprops}
The following properties hold.
\begin{enumerate}
    \item If $\lambda > 0$ then $\lambda I \succeq 0$.
    \item If $A \succeq 0$ then $B A B^\intercal \succeq 0$ for all $B$.
    \item If $A \succeq 0$ and $B \succeq 0$ then $A + B \succeq 0$.
    \item If $A \succ 0$ and $B \succeq 0$ then $A + B \succ 0$.
    \item If $A \succeq 0$ and $AB^\intercal \succeq BA B^\intercal$ then $A \succeq BA$.
    \item If $B \succ 0$ and $A^\intercal B A = 0$ then $A = 0$.
\end{enumerate}
\end{lemma}
\begin{proof}
Property 5 follows from Properties 1 and 2 and
$$A - BA = (B - I) A (B - I)^\intercal + (AB^\intercal - B A B^\intercal).$$
The other properties are standard exercises; see Appendix C of \citep{bos} for a full treatment.
\end{proof}
\end{section}

\begin{section}{Denoising and contractive autoencoders}
\label{denoising}

Here we connect regularized LAEs to the linear case of denoising (DAE) and contrastive (CAE) autoencoders.

A linear DAE receives a corrupted data matrix $\tilde{X}$ and is trained to reconstruct $X$ by minimizing
$$\mathcal{L}_{\text{DAE}}(W_1,W_2) = ||X - W_2W_1\tilde{X}||_F^2.$$
As shown in \citet{pretorius}, if $\tilde{X} = X + \epsilon$ is the corrupting process, where $\epsilon \in \mathbb{R}^{m\times n}$ is a noise matrix with elements sampled iid from a distribution with mean zero and variance $s^2$, then
$$\mathbb{E}\left[\mathcal{L}_{\text{DAE}}\right] = \frac{1}{2n} \sum_{i=1}^n ||x_i - W_2W_1x_i||^2 + \frac{s^2}{2}\tr(W_2W_1W_1^\intercal W_2^\intercal).$$
With $\lambda = ns^2$, we have
$$\mathbb{E}\left[\mathcal{L}_{\text{DAE}}\right] = \frac{1}{2n}\mathcal{L}_\pi.$$

The loss function of a linear CAE includes a penalty on the derivative of the encoder:
$$\mathcal{L}_{\text{CAE}}(W_1,W_2) = \mathcal{L}(W_1,W_2) + \gamma||J_f(x)||_F^2.$$
As shown in \citet{rifai}, if the encoder and decoder are \text{tied} by requiring $W_1 = W_2^\intercal$, then $\mathcal{L}_{\text{CAE}}$ equals $\mathcal{L}_\sigma$ with $\lambda = \frac{\gamma}{2}$:
$$\mathcal{L}_{\text{CAE}}(W_1) = \mathcal{L}_\sigma(W_1, W_1^\intercal).$$

\end{section}

\begin{section}{Further empirical exploration}

The Landscape Theorem also gives explicit forms for the trained encoder ${W_1}_*$ and decoder ${W_2}_*$ such that the matrices
\begin{equation}
    \label{transformations A and B}
    A = \Sigma_*^{-\frac{1}{2}}U^\intercal {W_2}_* \quad \text{and} \quad B = {W_1}_* U \Sigma_*^{-\frac{1}{2}}
\end{equation}
satisfy $AB = I_k$ for all losses and are each orthogonal for the sum loss. In Figure \ref{unit circle}, we illustrate these properties by applying the linear transformations $A$, $B$, and $AB$ to the unit circle $\mathbb{S}^1 \subset \mathbb{R}^2$.  Non-orthogonal transformations deform the circle to an ellipse, whereas orthogonal transformations (including the identity) preserve the unit circle.  This experiment used the same setup described in \ref{synetheic experiments} with $k = 2$.

\begin{figure}[tb]
\begin{center}
\begin{subfigure}{0.32\linewidth}
  \centerline{
  \includegraphics[width=\columnwidth]{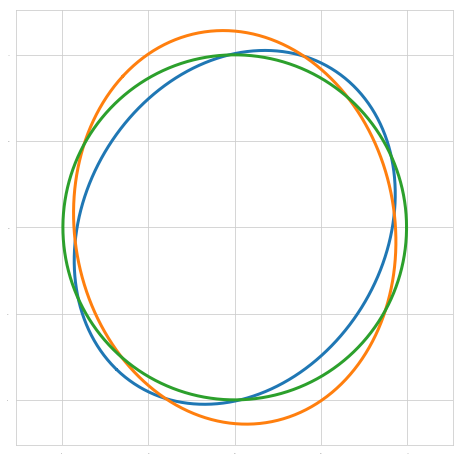}
  }
  \caption{Unregularized}
\end{subfigure}
\begin{subfigure}{0.32\linewidth}
  \centerline{
  \includegraphics[width=\columnwidth]{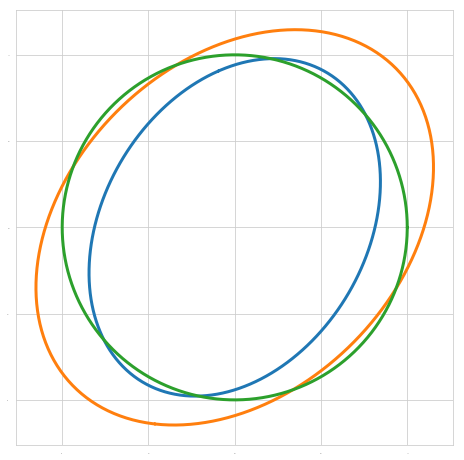}
  }
  \caption{Product}
\end{subfigure}
\begin{subfigure}{0.32\linewidth}
  \centerline{
  \includegraphics[width=\columnwidth]{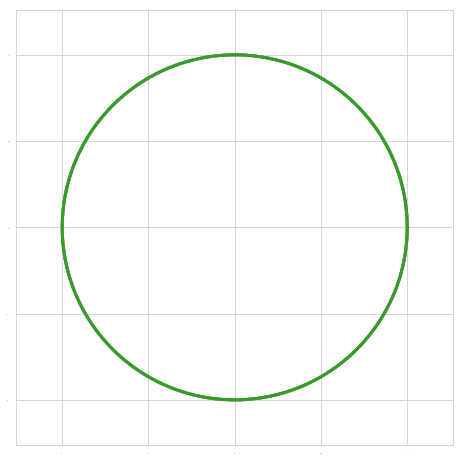}
  }
  \caption{Sum}
\end{subfigure}
\caption{Image of the unit circle (green) under $A$ (blue), $B$ (orange), and $AB$ (green) from (\ref{transformations A and B}). Non-orthogonal transformations deform the circle to an ellipse; orthogonal transformations preserve the circle.}
\label{unit circle}
\end{center}
\vskip -0.2in
\end{figure}

\subsection{MNIST}
\label{mnist experiments}
In the following experiment, the data set $X \in \mathbb{R}^{784 \times 10000}$ is the test set of the MNIST handwritten digit database \citep{lecun}. We train an LAE with $k=9$ and $\lambda=10$ for each loss, again using the Adam optimizer for $100$ epochs with random normal initialization, batch size of $32$, and learning rate $0.05$.

Figure \ref{weights} further illustrates the Landscape Theorem \ref{unified statement} by reshaping the left singular vectors of the trained decoder ${W_2}_*$ and the top $k$ principal direction of $X$ into $28 \times 28$ greyscale images. Indeed, only the decoder from the LAE trained on the sum loss has left singular vectors that match the principal directions up to sign.

\begin{figure}[tb]
\begin{center}
\begin{subfigure}{0.44\linewidth}
  \centerline{
  \includegraphics[width=0.85\columnwidth]{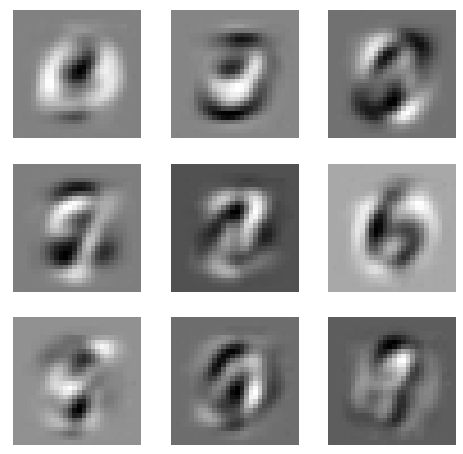}
  }
  \caption{Unregularized}
\end{subfigure}
\begin{subfigure}{0.44\linewidth}
  \centerline{
  \includegraphics[width=0.85\columnwidth]{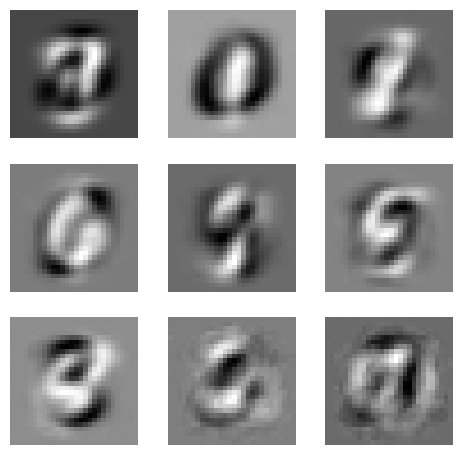}
  }
  \caption{Product}
\end{subfigure}
\begin{subfigure}{0.44\linewidth}
  \centerline{
  \includegraphics[width=0.85\columnwidth]{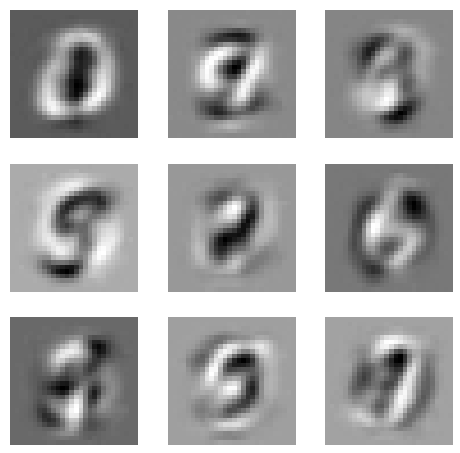}
  }
  \caption{Sum}
\end{subfigure}
\begin{subfigure}{0.44\linewidth}
  \centerline{
  \includegraphics[width=0.85\columnwidth]{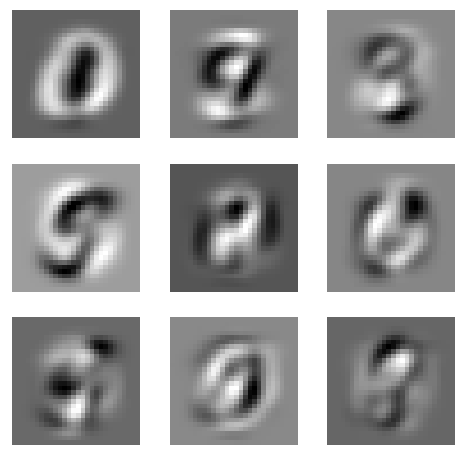}
  }
  \caption{PCA}
\end{subfigure}
\caption{Left singular vectors of the decoder from an LAE trained on unregularized, product, and sum losses and the principal directions of MNIST reshaped into images.}
\label{weights}
\end{center}
\vskip -0.2in
\end{figure}

As described in Section \ref{bayesian models}, for an LAE trained on the sum loss, the latent representation is, up to orthogonal transformation, the principal component embedding compressed along each principal direction. We illustrate this in Figure \ref{embedding} by comparing the $k=2$ representation to that of PCA.

\begin{figure}[tb]
\begin{center}
\begin{subfigure}{0.44\linewidth}
  \centerline{
  \includegraphics[width=\columnwidth]{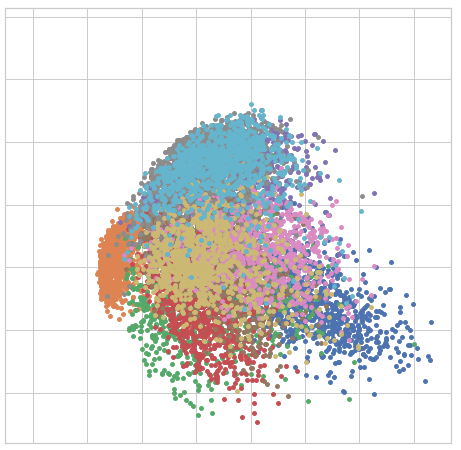}
  }
  \caption{Sum}
\end{subfigure}
\begin{subfigure}{0.44\linewidth}
  \centerline{
  \includegraphics[width=\columnwidth]{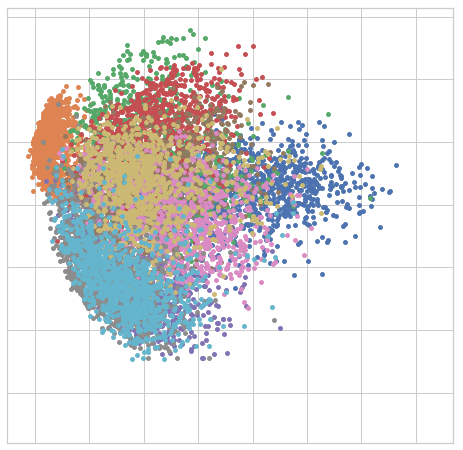}
  }
  \caption{PCA}
\end{subfigure}
\begin{subfigure}{0.10\linewidth}
  \centerline{
  \includegraphics[width=\columnwidth]{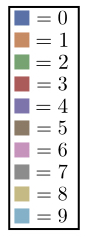}
  }
  \caption*{}
\end{subfigure}
\caption{Latent representations of MNIST learned by an LAE with sum loss and by PCA. Colors represent class label.}
\label{embedding}
\end{center}
\vskip -0.2in
\end{figure}

\end{section}

\begin{section}{Morse homology of the real Grassmannian}
\label{morse}

This section embraces the language and techniques of differential and algebraic topology to dive into the topology underlying LAEs. To complement Wikipedia, the following resources cover the italicized terminology in depth: \citep{milnor, hatcher,banyaga}.

Let $M$ be a smooth, compact manifold. In this section, we prove the Grassmannian Theorem through the lens of \textit{Morse theory}, a subfield of differential and algebraic topology that relates the topology of $M$ to smooth functions $f\colon M \to \mathbb{R}$.

A critical point of $f$ is \textit{non-degenerate} if the eigenvalues of the Hessian are non-zero. A \textit{Morse function} is a smooth function all of whose critical points are non-degenerate. Morse functions are generic and stable; see Section 4.2 of \citet{bloomthesis} for precise statements.

The \textit{Morse index} $d$ of a critical point is the number of negative eigenvalues of the Hessian. At each index-$d$ non-degenerate critical point, one can choose a local coordinate system under which the function takes the form
$$-x_1^2 - \ldots - x_d^2 + x_{d+1}^2 + \ldots + x_m^2.$$
Hence $d=0$ and $d=m$ correspond to parabolic minima and maxima, respectively, which all other values of $d$ correspond to saddles with, in local coordinates, $d$ orthogonal descending directions and $m-d$ orthogonal ascending directions. For example, the red, blue, and green critical points in Figure \ref{grassman} have Morse indices 0, 1, and 2, respectively.

The \textit{Morse inequalities} state that any Morse function on $M$ must have at least as many index-$d$ critical points as the \textit{Betti number} $b_d$, i.e.\ the rank of the \textit{singular homology group} $H_d(M; \mathbb{Z})$. This follows from a realization, called \textit{Morse homology}, of singular homology as the homology of a chain complex generated in dimension $d$ by the index-$d$ critical points. The boundary map $\partial$ counts negative gradient trajectories between critical points of adjacent index. A Morse function is {perfect} if this signed count is always zero, in which case $\partial$ vanishes. A Morse function is $\mathbb{F}_2$\textit{-perfect} if this count is always even, in which case $\partial$ vanishes over the field of two elements.

Not all smooth manifolds admit perfect Morse functions. For example, the projective plane $\mathbb{RP}^2 \cong \mathrm{Gr}_1(\mathbb{R}^3)$ cannot since $H_1(\mathbb{RP}^2; \mathbb{Z}) \cong \mathbb{F}_2$ implies that $\partial$ is non-zero. The \textit{Poincar\'e homology sphere} is a famous example of a manifold without a perfect Morse over $\mathbb{Z}$ or any field\footnote{This is because if a homology 3-sphere admits a perfect Morse function, then it consists of a 3-cell attached to a 0-cell, and is therefore the 3-sphere. Similarly, the \textit{smooth 4-dimensional Poincar\'e conjecture} holds if and only if every smooth 4-sphere admits a perfect Morse function. This conjecture, whose resolution continues to drive the field, states that there is only one smooth structure on the topological 4-sphere.}.

The Grassmannian $\mathrm{Gr}_k(\mathbb{R}^m)$ provides a coordinate-free representation of the space of rank-$k$ orthogonal projections, a submanifold of $\mathbb{R}^{m \times m}$. Through the identification of a projection with its image, $\mathrm{Gr}_k(\mathbb{R}^m)$ is endowed with the structure of a smooth, compact Riemannian manifold of dimension $k(m-k)$.

\begin{theorem}
$\mathcal{L}_X$ is an $\mathbb{F}_2$\textit{-perfect} Morse function. Its critical points are the rank-$k$ principal subspaces.
\end{theorem}
\begin{proof}
Consider the commutative diagram
\begin{equation}
\begin{tikzcd}[column sep = 5em]
\label{commutegrass}
    \mathrm{V}_k(\mathbb{R}^m) \arrow[r, "{\pi: O \mapsto \mathrm{Im}(OO^\intercal)}", two heads] \arrow[d, "{\iota: O \mapsto (O^\intercal,  O)}"', hook] &
    \mathrm{Gr}_k(\mathbb{R}^m) \arrow[d, "\mathcal{L}_X"] \\
    \mathbb{R}^{k\times m} \times \mathbb{R}^{m\times k} \arrow[r, "\mathcal{L}"'] &
    \mathbb{R}
\end{tikzcd}
\end{equation}
where $\mathrm{V}_k(\mathbb{R}^m)$ is the \textit{Stiefel manifold} of $m \times k$ matrices with orthonormal columns. Since $\iota$ is an immersion, by Theorem \ref{unified statement} the critical points of $\mathcal{L} \circ \iota = \mathcal{L}_X \circ \pi$ are all $k$-frames spanning principal subspaces of $X$. Since $\pi$ is a submersion, the critical points of $\mathcal{L}_X$ are the image of this subset under $\pi$ as claimed.

Each critical point (that is, rank-$k$ principle subspace) is non-degenerate because each of the included $k$ principal directions may be rotated toward any of the excluded $m-k$ principal directions in the plane they span, fixing all other principal directions; this accounts for all $k(m-k)$ dimensions. Flowing from higher to lower eigenvalues, these rotations are precisely the $-\nabla\mathcal{L}_X$ trajectories between adjacent index critical points. Since there are exactly two directions in which to rotate, we conclude that $\mathcal{L}_X$ is $\mathbb{F}_2$-perfect.
\end{proof}

While this paper may be the first to directly construct an $\mathbb{F}_2$-perfect Morse function on the real Grassmannian, the existence of some $\mathbb{F}_2$-perfect Morse function is straightforward to deduce from the extensive literature on perfect Morse functions on complex Grassmanians \cite{Hansen,duan}. Our simple and intuitive function is akin to that recently established for the special orthogonal group \citep{solgun}.

Note that $\mathcal{L}_X$ is invariant to replacing $X = U\Sigma V^\intercal$ with $U\Sigma$ and therefore doubles by replacing $X$ with the $2m$ points bounding the axes of the principal ellipsoid of the covariance of $X$. Rotating by $U^\intercal$, one need only consider the data set $$\{(\pm\sigma_1^2, 0, \ldots, 0), (0, \pm\sigma_2^2, \ldots, 0), \ldots, (0, 0, \ldots, \pm\sigma_m^2)\}$$ to appreciate the symmetries, dynamics, and critical values of the gradient flow in general.

We encourage the reader to check the Morse index formula \eqref{morseindex} in the case of $\mathrm{Gr}_2(\mathbb{R}^4)$ in the table below. The symmetries of the table reflect the duality between a plane and its orthogonal complement in $\mathbb{R}^4$.
\begin{table}[H]
\centering
\begin{tabular}{|c|c|c|c|c|}
    \hline 
    $d$ & $u_1$ & $u_2$ & $u_3$ & $u_4$ \\\hline\hline
    4 & & & $\bullet$ & $\bullet$ \\\hline
    3 & & $\bullet$ & & $\bullet$ \\\hline
    2 & & $\bullet$ & $\bullet$ & \\\hline
    2 & $\bullet$ & & & $\bullet$ \\\hline
    1 & $\bullet$ & & $\bullet$ &\\\hline
    0 & $\bullet$ & $\bullet$ & & \\\hline
\end{tabular}
\end{table}

Theorem \ref{grassman} implies that $\mathcal{L}_X$ endows $\mathrm{Gr}_k(\mathbb{R}^m)$ with the structure of a \textit{CW complex}; each index-$d$ critical point $P$ is the maximum of a $d$-dimensional cell consisting of all points that asymptotically flow up to $P$. This decomposition coincides with the classical, minimal CW construction of Grassmannians in terms of \textit{Schubert cells}. Over $\mathbb{Z}$, pairs of rotations have the same sign when flowing from even to odd dimension, and opposite signs when flowing from odd to even dimension, due to the oddness and evenness of the antipodal map on the boundary sphere, respectively. In this way, Morse homology for $\mathcal{L}_X$ realizes the same chain complex as \textit{CW homology} on the Schubert cell structure of the real Grassmannian.

\subsection{Morse homology and deep learning}

We have seen how the rich topology of the real Grassmannian forces any generic smooth function to have at least $\binom{m}{k}$ critical points. More interestingly from the perspective of deep learning, Morse homology also explains why simple topology forces critical points of any generic smooth function to ``geometrically cancel'' through gradient trajectories. As an intuitive example, consider a generic smooth function $f: \mathbb{R} \to \mathbb{R}$ that is strictly decreasing for $x < a$ and strictly increasing for $x > b$ for some $a < b$. Then on $[a, b]$, $f$ wiggles up and down, alternating between local minima and maxima, with pairwise gradient cancellation leaving a single minimum.

More generally, for a generic smooth loss function over the connected parameter space $\mathbb{R}^p$, diverging strictly to infinity outside of a compact subset, each pair of minima is linked by a path of gradient trajectories between minima and index-1 saddles\footnote{For example, in Figure \ref{one dimensional}(c), the red minima are each connected to the yellow saddle by one gradient trajectory.}. In fact, since $\mathbb{R}^p$ is contractible, we can flow upward along gradient trajectories from one minimum to all minima through index-1 saddles, from those index-1 saddles to other index-1 saddles through index-2 saddles, and so on until the resulting chain complex is contractible. Note there may exist additional critical points forming null-homotopic chain complexes.

The contractible complex containing the minima is especially interesting in light of \citet{Choromanska}. For large non-linear networks under a simple generative model of data, the authors use random matrix theory\footnote{From this perspective, the LAE is a toy model that more directly bridges loss landscapes and random matrix theory; the heights of critical points are sums of eigenvalues.} to prove that critical points are layered according to index: local minima occur at a similar height as the global minimum, index-1 saddles in a layer just above the layer of minima, and so on. Hence Morse homology provides a principled foundation for the empirical observation of low-lying valley passages between minima used in Fast Geometric Ensembling (FGE) \citep{Garipov}.

In FGE, after descending to one minimum, the learning rate is cycled to traverse such passages and find more minima. While the resulting ensemble prediction achieves state-of-the-art performance, these nearby minima may correspond to models with correlated error. With this in mind, we are exploring whether ensemble prediction is improved using less correlated minima, and whether many such minima may be found with logarithmic effort by recursively bifurcating gradient descent near saddles to descend alongside the Morse complex described above.
\end{section}

\newpage

\section*{Acknowledgements}

We thank Mehrtash Babadi, Daniel Bear, Yoshua Bengio, Alex Bloemendal, Bram Gorissen, Matthew J. Johnson, Scott Linderman, Vardan Papyan, Elad Plaut, Tomaso Poggio, and Patrick Schultz for helpful discussions.



\nocite{Bengio}

\bibliography{sample}
\bibliographystyle{icml2019}





\end{document}